%% file: main.tex
\documentclass{article}
\usepackage[nonatbib,final]{neurips_2021}
\usepackage{amsmath,amsthm,fancyhdr,lineno,wrapfig}
\input{headers.tex}

\input{defs.tex}

\newcommand\blfootnote[1]{%
  \begingroup
  \renewcommand\thefootnote{}\footnote{#1}%
  \addtocounter{footnote}{-1}%
  \endgroup
}

\usepackage{listings} %
\usepackage{enumitem}
\usepackage{mathrsfs}
\usepackage[utf8]{inputenc} %
\usepackage[T1]{fontenc}    %
\usepackage{booktabs}       %
\usepackage{amsfonts}       %
\usepackage{nicefrac}       %
\usepackage[hide=true,setmargin=true,marginparwidth=1.2in]{marginalia}

\title{%
Deep Learning on a Data Diet: \\
Finding Important Examples Early in Training
}
\author{%
  Mansheej Paul %
  \\
  Stanford University\\
  \texttt{mansheej@stanford.edu} \\
  \And
  Surya Ganguli \\
  Stanford University; Facebook AI Research \\
  \texttt{sganguli@stanford.edu} \\
  \AND
  Gintare Karolina Dziugaite \\
  Mila \thanks{This work was carried out while the author was at ServiceNow. It was finalized at Google Brain.}\\
  \texttt{gkdz@google.com} \\
}
\graphicspath{ {./figures/} }
\begin{document} 
\newcommand{\objective}[1]{\EmpRisk{S}{#1}}

\maketitle
\begin{abstract}
Recent success in deep learning has partially been driven by training increasingly overparametrized networks on ever larger datasets. It is therefore natural to ask: how much of the data is superfluous, which examples are important for generalization, and how do we find them? 
In this work, we make the striking observation that, in standard vision datasets, simple scores averaged over several weight initializations can be used to identify important examples \textit{very early in training}.
We propose two such scores---the Gradient Normed (GraNd) and the Error L2-Norm (EL2N) scores---and demonstrate their efficacy on a range of architectures and datasets by pruning significant fractions of training data without sacrificing test accuracy.
In fact, using EL2N scores calculated a few epochs into training, we can prune half of the CIFAR10 training set while slightly improving test accuracy.
Furthermore, for a given dataset, EL2N scores from one architecture or hyperparameter configuration generalize to other configurations.
Compared to recent work that prunes data by discarding examples that are rarely forgotten \emph{over the course of training}, our scores use only \emph{local information early in training}.
We also use our scores to detect noisy examples and study training dynamics through the lens of important examples---we investigate how the data distribution shapes the loss surface and identify subspaces of the model's data representation that are relatively stable over training.

\end{abstract}

\blfootnote{This version supersedes the published NeurIPS 2021 version of this work. 
Due to a bug in Flax\citep{flax2020github} identified by \citet{andreas}, the
results for the GraNd score computed at initialization were miscalculated, and subsequent conclusions about pruning at initialization were erroneous. 
This version corrects these errors.}
\section{Introduction}
Recently, deep learning has made remarkable progress driven, in part, by training over-parameterized models on ever larger datasets. 
This trend creates new challenges: the large computational resources required pose a roadblock to the democratization of AI.
Memory and resource constrained settings, such as on-device computing, require smaller models and datasets. 
Identifying important training data plays a role in online and active learning.
Finally, it is of theoretical interest to understand how individual examples and sub-populations of training examples influence learning.

To address these challenges, we propose a scoring method that can be used to identify important and difficult examples early in training, and prune the training dataset without large sacrifices in test accuracy. We also investigate how different sub-populations of the training data identified by our score affect the loss surface and training dynamics of the model.

Recent work on pruning data \citep{coleman2019selection,hwang2020data}, can be placed in the broader context of identifying coresets---examples that provably guarantee a small gap in training error on the full dataset \citep{har2007smaller,huggins2016coresets,campbell2018bayesian,tsang2005core,munteanu2018coresets}. However, due to the nonconvex nature of deep learning, coreset techniques make conservative estimates that lead to weak theoretical guarantees and are less effective in practice.

A different approach recently proposed by \citet{toneva2018empirical} tracks the number of times through training an example transitions from being correctly classified to misclassified, called a "forgetting event", and find that some examples are rarely forgotten, while others are forgotten repeatedly. Empirically, they observed that training accuracy is not affected by the rarely forgotten training \samples and a large fraction of the \trainingdata can be removed without any impact on test accuracy. However, since this method relies on collecting forgetting statistics throughout training, the forgetting score is typically calculated in the middle of or at the end of training. \citet{toneva2018empirical} find that, in their example of a ResNet18 trained on CIFAR-10 for 200 epochs, the Spearman rank correlation between early and late scores is good after about 25 epochs and stabilizes after 75 epochs. 

Broadly speaking, the ability to prune datasets raises a number of questions: What is the nature of \samples that can be removed from the \trainingdata without hurting accuracy? How early in training can we recognize such \samples? How many \samples do we need and how does this depend on the data distribution? 
These questions may have no generic answers and so, in this work, we begin to pursue them empirically in the context of several standard vision benchmarks  and standard network architectures.
Answers to these questions may both (1) lead to new methodologies that could dramatically reduce training times and memory requirements, and (2) offer important insights into the training dynamics of deep neural networks, and the role of data.

Our first finding is that \emph{very early in training} (just a few epochs), partial forgetting scores identify large fractions of data that can be pruned. Analyzing this puzzling result with a one gradient step analysis of training suggests a very simple heuristic: use the loss gradient norm of individual \samples to identify important \samples. 
While this approach does not work when the loss gradient norms are computed at the weights early in training of a single trajectory, we find that, surprisingly, averaging these norms over multiple weight initializations does produce a ranking that correlates strongly with forgetting scores and allows us to prune a significant fraction of examples early in training.
Indeed, we can prune 50\% of examples from CIFAR-10 without affecting accuracy, while on the more challenging CIFAR-100 dataset, we can prune 25\% of examples with only a 1\% drop in accuracy.

Through a series of empirical studies, we have begun to tease apart the properties of important examples and how they can depend on the data distribution. In particular, we find that the examples with the very highest norms become superfluous as the amount of label noise increases. Indeed, even on clean data, we find that in the high pruning regime, the best population excludes the very highest-scoring \samples.

\subsection{Contributions}

\begin{itemize}[leftmargin=2em]

\item We propose to score the importance of each training \sample $(x_i,y_i)$  by its expected loss gradient norm (\scorename score), which, up to a constant, bounds the expected change in loss for an arbitrary \sample $(x,y)$ caused by removing $(x_i,y_i)$.

\item Our experimental findings suggest that, within the first few epochs of training, the \scorename score is well-approximated  by the norm of the error vector (\scoreerror score), 
where the error vector is the predicted class probabilities minus one-hot label encoding.
In fact, we find that the \scoreerror score provides an even stronger signal for data-pruning---for CIFAR10 we can prune 50\% of the data, and for the harder CIFAR100, we can prune as much as 25\% of the data without any loss in test accuracy.

\item We study the role of \samples with the highest \scoreerror scores, and find that excluding a small subset of the very highest scoring \samples produces a boost in performance. This boost in performance is enhanced in a corrupted label regime.

\item We introduce a method, based on linearly connected modes, for studying the empirical risk surface in terms of the modes of \emph{subsets of data}, allowing us to identify when, in training, the final performance on subpopulations is determined.
We demonstrate that the linearly connected mode at-convergence of empirical risk surface computed on low \scoreerror score \samples is determined much earlier in training compared to high score \samples. 

\item Finally, we study how an \sample's \scoreerror score connects to the network's training dynamics. We do so by tracking the data-dependent NTK submatrices corresponding to the low or high score \samples, and measuring the rate at which it evolves in a scale-invariant way. We find that the NTK submatrix for the high score \samples evolves faster throughout training, supporting our hypothesis that high-scoring \samples are the ones driving the learning and the changes in the NTK feature space \citep{fort2020deep}.

\end{itemize}

\section{Which samples are important for learning?}

\subsection{Preliminaries}
\label{notation}
We consider supervised classification,
where $S = \{(x_i, y_i)\}_{i=1}^N$ denotes the training set,
drawn \iid from an unknown data distribution $\Dist$,
with input vectors $x \in \Reals^d$ and one-hot vectors $y \in \{0,1\}^K$ encoding labels.
For a fixed neural network architecture,
let $f_\weights(x) \in \Reals^K$ be the logit outputs of the neural network with
weights $\weights \in \parspace \subseteq \Reals^D$ on input $x \in \Reals^d$. 
Let $\sigma$ be the softmax function given by $\sigma(z_1,\dots,z_K)_k = \exp \{z_k\} / \sum_{k'=1}^{K} \exp\{z_{k'}\}$.
Let $p(\weights,x) = \sigma(f(\weights,x))$ denote the neural network output in the form of a probability vector.
For any probability vector $\hat p$,
let $\loss(\hat p, y) = \sum_{k=1}^{K} y^{(k)} \log \hat p^{(k)}$ denote cross-entropy loss.

Let $\weights_0,\weights_1,\weights_2,\dots,\weights_T$ be the iterates of stochastic gradient descent (SGD),
where, for some sequence of minibatches $S_0,S_1,\dots,S_{T-1} \subseteq S$ of size $M$,
we have 
\[\textstyle
\weights_{t} = \weights_{t-1} - \eta  \sum_{(x,y) \in S_{t-1}} g_{t-1}(x,y),
\]
for $g_{t-1}(x,y) = \grad_{\weights_{t-1}} \loss(p(\weights_{t-1},x),y)$, and $t = 1,\dots,T$.

\subsection{Gradient Norm Score and an infinitesimal analysis}

Fix a training set $S$. Due to training with SGD from a random initialization, the weight vector at time $t>0$, $\weights_t$, is a random variable. 
The expected magnitude of the loss vector is our primary focus:
\begin{definition} 
\label{def:score}
The \scorename score of a training \sample $(x,y)$ at time $t$ is \rlap{$
 \score{t}(x,y)
  = \EE_{\weights_t} \left\|g_{t} (x,y) \right\|_2.
  $}
\end{definition}

Here we 
describe conditions under which  the \scorename score controls the contribution of a training \sample to the change in the training loss.
In order to simplify our analysis, we approximate the training dynamics as if they were in continuous time. %

A key quantity in our analysis is the time derivative of the loss for a generic labeled \sample $(x,y)$:
$
\Delta_t ((x,y),S_t) = -\frac{\dee \loss(f_t(x), y) }{\dee t}$ (where $f_t(\cdot) = f_{\weights_t}(\cdot)$), 
i.e., the instantaneous rate of change in the loss on $(x,y)$ at time $t$,
where the gradient is computed on the minibatch $S_t$.
By the chain rule,
\[
\textstyle
\Delta_t ((x,y),S_t) 
=  g_t(x,y)
\frac{\dee \weights_{t}}{\dee t}.
\]
This relates to our discrete time dynamics via $\frac{\dee \weights_{t}}{\dee t} \approx \weights_{t+1} - \weights_{t} = - \eta  \sum_{(x',y') \in S_{t}} g_{t} (x',y')$.

\newcommand{\tp}{(x^*,y^*)}
Our goal is to understand how removing a training point from minibatch $S_t$ affects $\Delta_t (\tp,S_t)$ for any $\tp$. If a training point $(x,y)$ is not in the minibatch $S_t$, then the effect is trivial.
We thus study $\Delta_t (\tp,S)$ in order to be able to rank all the training examples.

\begin{lemma}
\label{lemmadeltadiff}
Let $S_{\neg j} = S \setminus (x_j,y_j)$. 
Then for all $\tp$, there exists $c$ such that
\[
 \|\Delta_t (\tp, S) - \Delta_t (\tp, S_{\neg j}) \| \leq c  \| g_t(x_j,y_j) \| .
 \label{eq:onestepbound}
\]
\end{lemma}

\begin{proof}
For a given \sample $\tp$, the chain rule yields
$
\Delta_t (\tp,S) = -\frac{\dee \loss(f_t(x^{*}),y^*) }{\dee t} =  
\frac{\dee \loss(f_t(x^*),y^*)}{ \dee \weights_t}
\frac{\dee \weights_{t}}{\dee t}.
$
Since the weights are updated using SGD, we have $\frac{\dee \weights_{t}}{\dee t} = -\eta  \sum_{(x_j,y_j) \in S_t} g_{t}(x_j,y_j)$.
Letting $c = \eta \| \frac{\dee \loss(f_t(x^*),y^*)}{ \dee \weights_t}\|$, 
the result follows.
\end{proof}

At any given training step, given the current location $\weights_t$, the contribution of a training \sample $(x,y)$\footnote{Here we drop the index $j$ since we refer to an arbitrary training point.} to the decrease of loss on any other \sample, is bounded by \cref{eq:onestepbound}. 
Since the constant $c$ does not depend on the training example $(x,y)$\footnote{Note that $c$ depends on $\tp$ but for a given $\tp$, $c$ is fixed for all training inputs $(x,y)$, allowing us to rank the training \samples.
}, we only consider the gradient norm term, $\|g_{t}(x,y)\|$.
The expected value of this gradient norm
is exactly the \scorename score of $(x,y)$. 
In other words, \samples with a small \scorename score in expectation have a bounded influence on learning how to classify the rest of the training data at a given training time\footnote{The opposite is not  necessarily true: \samples with large scores may have gradients that cancel out and do not contribute much, meaning that this upper bound is loose.}.
We therefore propose to rank training \samples by their \scorename scores, larger norm meaning more important for preserving $\Delta_t(x)$.

For an arbitrary input $x \in \Reals^d$, 
let $\feature_t^{(k)}(x) = \grad_{\weights_t} f_t^{(k)}(x)$ denote the $k$th logit gradient. Then \scorename can be written as\footnote{The score $\score{t}(x,y)$ is a function of $(x,y)$. Thus $(x,y)$
is non-random, and the expectation is taken over the remaining randomness (the weights at time t which depend on a random initialization, random minibatch sequence, GPU noise, etc.).}
\[\textstyle
 \score{t}(x,y)
  = \EE  \left\|  \sum_{k=1}^{K} \grad_{f^{(k)}} \loss (f_t(x),y)^T \feature_t^{(k)}(x)  \right\|_2.
  \label{eq:multiclassscoredefn} 
\]
Under the cross entropy loss, $\grad_{f^{(k)}} \loss (f_t(x),y)^T  = p(\weights_t,x)^{(k)} - y_k$. When $\{\feature_t^{(k)}(x)\}_{k}$ are roughly orthogonal across logits, and are of a similar size across logits and training \samples $x$, then we can approximate \scorename by just the norm of the error vector.

\begin{definition} 
The \scoreerror score of a training sample $(x,y)$ is defined to be 
$\EE \| p(\weights_t,x) - y \|_2$.
\end{definition}

Our experimental results suggest that this approximation becomes accurate after a few epochs of training (see \cref{sec:emp}).
These approximations are also in agreement with the empirical results reported in \citep{fort2019emergent,fort2020deep}. 
\citet[Sec.~5.1]{fort2019emergent} demonstrate that the mean logit gradients are nearly orthogonal among classes throughout training.
The authors demonstrate that per-example gradients cluster around the mean logit gradient.
\citet[Figs.~12D-14D]{fort2020deep} provide evidence that the mean logit gradient directions evolve rapidly early in training and then stabilize (as measured by the cosine distance between mean logit gradient vectors at different times in training).

\subsection{Comparison to forgetting scores}
\label{exp:comparisontoforget}

\citet{toneva2018empirical} define a ``forgetting event'' for a training sample to be a point in training when the classifier switches from making a correct classification decision to an incorrect one.
They define an approximate \emph{forgetting score} for each training \sample as the number of times during training when it was included in a minibatch {\it and} underwent a forgetting event.
\citeauthor{toneva2018empirical} demonstrate that \samples with low forgetting score may be completely omitted during training without any noticeable effect on the accuracy of the learned predictor. 
In \cref{fig:tradeoffs} and \cref{app:correlation}, we make an empirical comparison of forgetting scores to our proposed \scorename and \scoreerror scores.

In \cref{lemmadeltadiff}, we bounded the contribution of a training \sample to the decrease of the loss of any other sample over a single gradient step. 
Due to $\feature_t(\cdot)$'s being time-dependent, it is complicated to extend the analysis to multiple steps.
However, it is interesting to consider a case when $\feature_t(x_i) = \feature(x_i)$ for all $x_i$ in the training set, and $K=1$. 
Then summing the bound in \cref{eq:onestepbound} on how much a sample $(x_j,y_j)$ affects the logit output on an arbitrary point at each time $t \in \{1,..,T\}$, 
we obtain a score that depends on
$\| \feature(x_j) \| |\sum_t (p_t(x_j)-y_j) | $.
For two \samples, $(x,y)$ and $(x',y')$, such that $\| \feature(x') \| \approx \| \feature(x) \|$ ,
we see that the \sample that is learned faster and maintains small error over training time will have a smaller \scorename score on average throughout training. 
Note that $|(p_t(x_j)-y_j) |$, if rescaled, is an upper bound on 0--1 loss, and therefore $\sum_t |(p_t(x_j)-y_j) | $ upper bounds the number of forgetting events during training (after rescaling). 
In this simplified setting an \sample with a high number of forgetting events will also have a high \scorename score.

\section{Empirical Evaluation of \scorename and \scoreerror Scores via Data Pruning}
\label{sec:emp}

\begin{figure}[t]
\includegraphics[width=\linewidth]{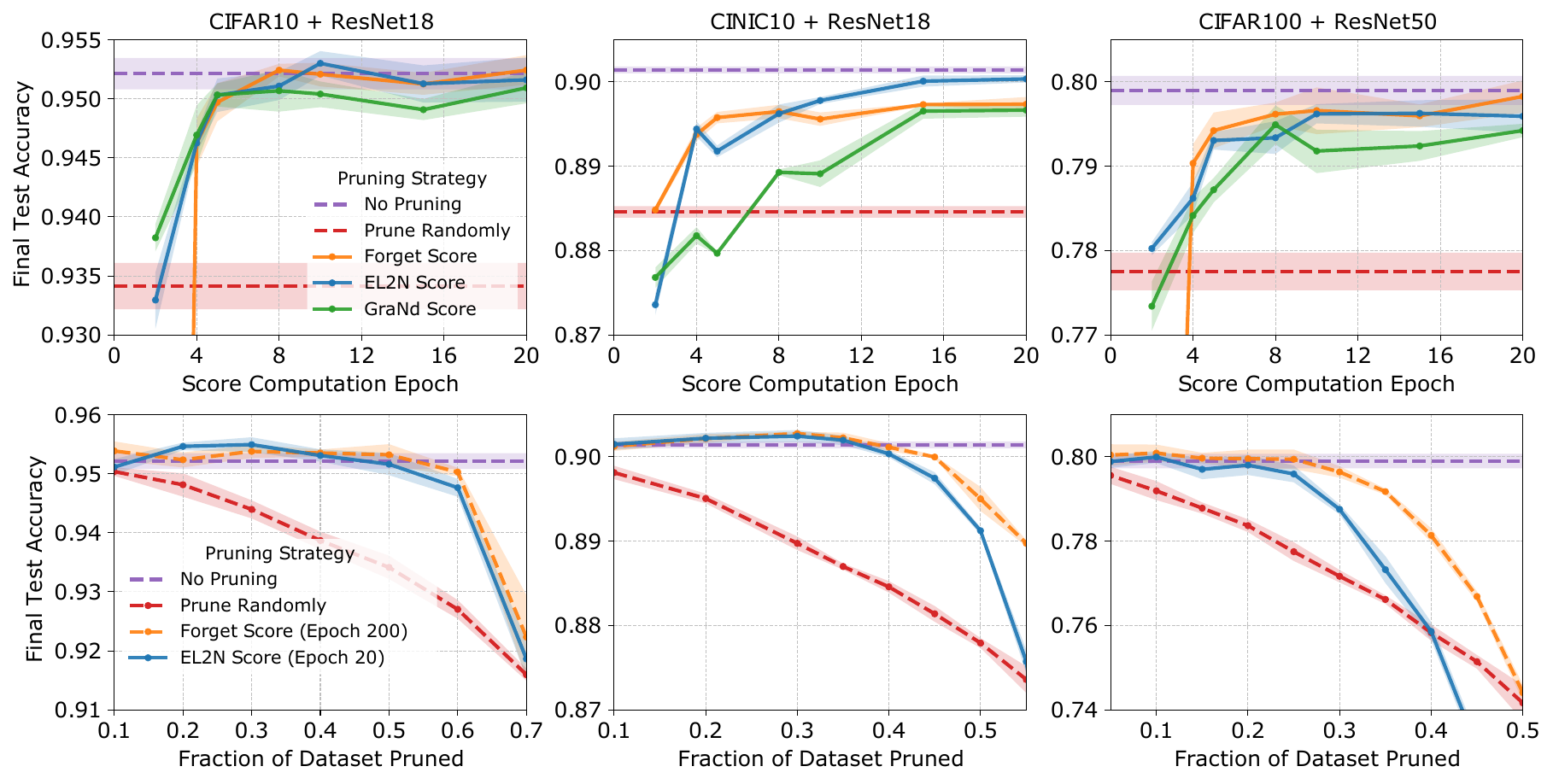}
\centering 
\caption{Columns correspond to three different dataset and network combinations (labeled at the top). 
Each legend applies to all 3 figures in its row.
\textit{First row:} Final test accuracy achieved by training on a subset of training data comprised of examples with maximum forgetting, \scoreerror and \scorename scores computed at different times early in training. 
Subsets of a fixed size are used: networks are trained on 50\% of training data for CIFAR-10, 60\% for CINIC-10 and 75\% for CIFAR-100. 
\textit{Second row:} Final test accuracy achieved by training after different fractions of the dataset are pruned. Here we compare forgetting scores at the end of training and \scoreerror scores early in training (at epoch 20).
In each case, examples with the lowest scores are pruned at initialization.
\textit{In all experiments} accuracies achieved by training on the full dataset and on a random subset of the corresponding size are used as baselines.}
\label{fig:tradeoffs}
\end{figure}

In the previous section, we motivated \grand and \eltoon scores 
by quantifying the influence of a training example on the loss of an arbitrary example after one optimization step.
In this section, we evaluate these scores empirically, and verify that they identify examples important for generalization.
Networks trained on subsets of the data with high scores achieve levels of test accuracy comparable to training on the full dataset and are competitive with other state of the art data pruning methods. 
Perhaps most remarkably, these scores are effective even when computed early in training.

\textbf{Data pruning experiments.} 
We train convolutional neural networks of varying depth--ResNet18 and ResNet50 \citep{he2016deep}--on standard vision datasets of varying difficulty--CIFAR-10, CIFAR-100 \citep{krizhevsky2009learning}, and CINIC-10 \citep{cinic10cite}. 
All scores are calculated by averaging the scores from ten independent training runs.
After calculating scores and selecting a training subset, final test accuracies are obtained by retraining networks from new random initializations on only the selected subset. 
Networks used for evaluating the scores are initialized with seeds that are different from those used to calculate the scores.
For each experiment, we report the mean of four independent runs and represent variability across runs by shading the region
which spans the 16th to 84th percentile of obtained accuracies. See \cref{app:params} for more implementation details and \cref{app:additional} for additional experiments. 

In \cref{fig:tradeoffs}, we show the results of two sets of experiments (top and bottom) on three different network and dataset combinations. 
The first experiment asks, how early in training are forgetting, \grand and \eltoon scores effective at identifying examples important for generalization?
We compare the final test accuracy from training on subsets of fixed size but pruned based on scores computed at different times early in training.
The second experiment compares how \scorename 
and \scoreerror scores early in training and forgetting scores at the end of training negotiate the trade-off between generalization performance and training set size. 
The training sets are constructed by pruning different fractions of the lowest score examples. 
In all examples, training on the full dataset and a random subset of the corresponding size are used as baselines. We make the following observations.

\textbf{Pruning at initialization.}
In all settings, \scorename scores can be computed at initialization and used to select a training subset.
Unfortunately, the scores at initialization do not perform well.
An earlier version of this paper reported some surprising results on pruning at initialization. At subsequent work by \citet{andreas}, these results were found to be erroneous due to a bug in Flax v0.3.3 \footnote{See \url{https://github.com/google/flax/commit/28fbd95500f4bf2f9924d2560062fa50e919b1a5}}.

\textbf{Pruning early in training.}
We find that, after only a few epochs of training, \scoreerror scores are extremely effective at identifying important examples for generalization. For a wide range of intermediate pruning levels, training on the highest scores performs on par with or better than training on the full dataset. Even at higher pruning levels, \scoreerror scores computed using local information early in training are competitive with forgetting scores which integrate information over the training trajectory. This suggests that the average error vector \emph{a few epochs into training} can identify \samples that the network heavily uses to shape the decision boundary \emph{throughout training}.

Interestingly, at extreme levels of pruning with either \scoreerror or \scorename scores, we observe a sharp drop in performance. We hypothesize that this is because at high levels of pruning, using either \scorename or \scoreerror scores leads to bad coverage of the data distribution.
By only focusing on the highest error \samples, it is likely that an entire subpopulation of significant size that is present in the test data is now excluded from the training set.
We only fit a small number of very difficult \samples and do not keep enough of a variety of \samples for  training models with good test error.

\textbf{A property of the data.}
Our results suggest that the ranking of important examples induced by \eltoon scores is a property of the dataset and not specific to a network.
First, in \cref{app:resnet50cifar10}, we show that a ResNet18 and a ResNet50 trained on CIFAR-10 have similar performance curves and the same amount of data can be pruned, even though ResNet50 is a much deeper network with more parameters. 
Second, \eltoon scores calculated on one set of network architecture and hyperparameter configurations can be used to prune data for training with a different network architecture or hyperparameter configuration.
The set of important examples generalizes across architectures and hyperparameters. See \cref{app:arch_gen} for the experiment on generalization across architectures and \cref{app:hpo} for the experiment on using scores calculated during hyperparameter optimization.
Additionally, in an analysis of the sensitivity of the scoring methods to hyperparameters in \cref{app:details}, we observe that scores calculated on a single network do not perform as well as those averaged across networks.

We hypothesize that averaging the gradient or error norms over multiple initializations or training trajectories removes dependence on specific weights, allowing a more accurate distillation of the properties of the dataset. EL2N scores can thus be used to probe and understand how the distribution of the training data impacts dynamics (as we show in the next sections). Additionally, it can also reduce the computational burden of training neural networks; once we compute the scores, future networks can be trained on the pruned dataset.

In the following experiments, we focus on \scoreerror scores computed early in training, as they appear to more accurately identify important examples.

\begin{wrapfigure}{r}{0.66\textwidth}
\includegraphics[width=\linewidth]{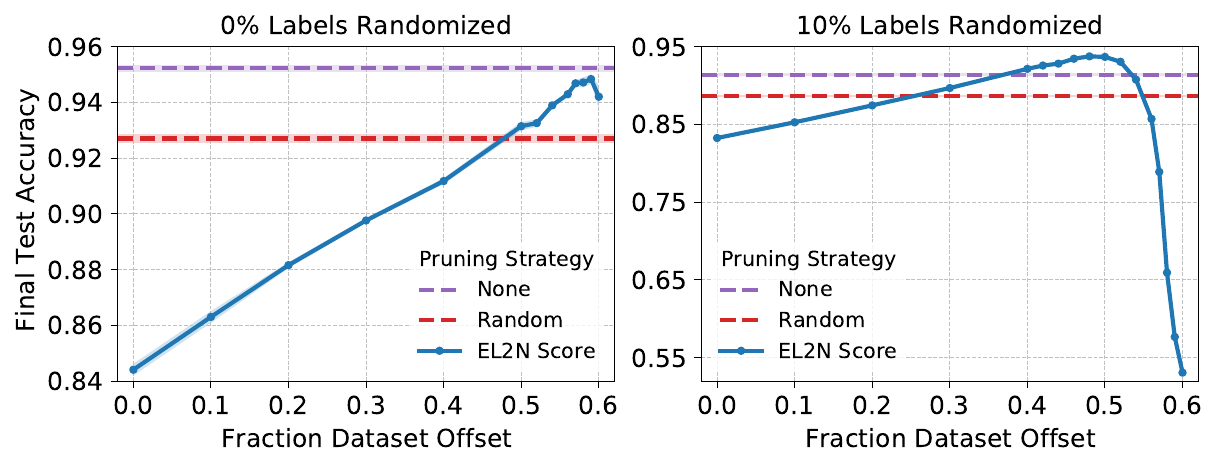}
\centering 
\caption{ResNet18 trained on a 40\% subset of CIFAR-10 with clean \textit{(left)} and 10\% randomized labels \textit{(right)}. 
The training subset contains the \emph{lowest} scoring \samples
\emph{after} \samples with scores below the offset are discarded.
Scores computed at epoch 10.
}
\label{fig:slidingwindow}
\end{wrapfigure}
\vspace{-1em}

\section{Identifying noise examples}
\label{app:noisydata}

In the previous section, we studied the effect of keeping the highest-scoring \samples,
and found that we could train on only the top 50\% of \samples by score without a drop in accuracy (CIFAR-10).
What is the nature of subpopulations of \samples that allow us to reach high accuracy?
One hypothesis is that the highest-scoring \samples are the most important ones for achieving an accurate classifier. In this section, we refute this hypothesis, and demonstrate the role of label noise.

To test whether the highest-scoring \samples are most important for achieving high accuracy, we first sort the \samples by increasing \scoreerror score computed after a small number of training epochs.
Then we perform a sliding window analysis by training on a subset of examples with scores within a window from percentile $f$ to percentile $f+P$ percentile, always keep $P$\% of the data but sliding up $f$. As this window slides to higher percentiles, performance increases, except when the window includes examples with the very highest scores \cref{fig:slidingwindow} (left). Indeed the the optimal sliding window actually excludes approximately $500$ of the highest-scoring training \samples. These effects are reduced in the low pruning regime (see \cref{app:noise}).
In \cref{app:imgs}, we visualize some of the images that are excluded from each class.

Before we analyze these results, we first place them into a wider context, 
where we also change the amount of noise in the underlying label distribution.
We repeat the experiment outlined above, but corrupt a random $K$\% of labels, replacing them with a random label, mirroring the protocol popularized by \citet{zhang2016understanding}.
\cref{fig:slidingwindow} reveals that with increased label corruption, 
the optimal window shifts and excludes a higher number of \samples.
Therefore, the effect we see in the noiseless case appears to be
magnified in the presence of label noise.
\cref{app:noisescore} examines how adding label noise influences the distribution of \eltoon scores of examples.

These findings have several implications.
The most obvious implication is that training with only the highest-scoring samples may not be optimal, especially when there is label noise.  
\NA{When the population has a low Bayes error rate,
using only the highest scoring samples yields optimal results.}
However, without a validation set, one should be cautious in excluding high-score \samples.
\citet{feldman2020does} discusses memorization in a noisy-label setup
and gives conditions under which one should memorize 
in order to not misclassify singleton \samples ( \samples in the \trainingdata that are the sole representatives of a subpopulation). 
For example, if the subpopulation appears with a frequency $\Omega(1/N)$,
memorizing such \samples can improve generalization.
In practice, we may not know whether our data fits these conditions. However, our analysis in \cref{fig:slidingwindow} suggests a simple and powerful method to prune data for optimal performance by optimizing just two hyperparameters of a sliding window using a validation set. 

\section{Optimization landscape and the training dynamics}
\label{app:trainingdynamics}

\subsection{Evolution of the data-dependent NTK}
\label{sec:ntk}

The dynamics of neural-network training in the infinite-width limit 
are now well understood \citep{jacot2018neural,lee2019wide}:
for an appropriate scaling of the learning rate and initial weights, 
the neural network behaves like a linear model in which the data is transformed by the Neural Tangent Kernel (NTK) at initialization, which is defined as 
the product of
the Jacobians of the logits at initialization. In the limit, neural network training implements kernel regression with the fixed NTK as the kernel. 

\begin{wrapfigure}{r}{0.66\textwidth}
\includegraphics[width=\linewidth]{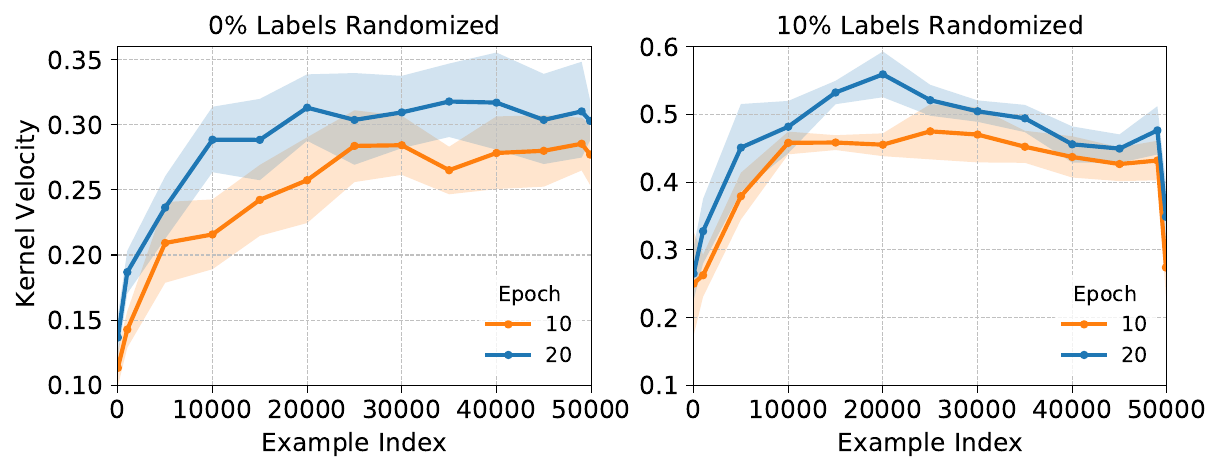}
\centering 
\caption{Kernel velocity for different subsets of images when ResNet18 is trained on CIFAR-10 with all true labels \textit{(left)} and 10\% label noise \textit{(right)}. Examples are sorted in ascending order by \scoreerror scores and each point corresponds to the kernel velocity of 100 contiguous images starting at example index. Both scores and velocities are computed at the same epoch indicated by color.}
\label{fig:ntk}
\vspace{-2em}
\end{wrapfigure}

However, finite neural networks outperform their infinite-width limits \citep{arora2019harnessing} and have different dynamics early in training \citep{lewkowycz2020large}. In fact, rather than being constant, the data-dependent NTK, defined by \citet{fort2020deep} as the Gram matrix of the logit Jacobian, evolves with high velocity in the initial phase of training.
Then, around the time of onset of linear mode connectivity, the NTK velocity stabilizes at a smaller value and remains nearly constant for the rest of the high learning rate training time.

Here we seek to understand which training samples contribute to the NTK gram matrix evolution.
To empirically approximate the velocity of a NTK submatrix corresponding to a subset of images in a scale invariant way, we follow  \citep{fort2020deep}. We compute the cosine distance between two NTK gram matrices on the given subset, one computed at epoch $t$, and another one at epoch $t+1$, one epoch later (see \cref{kvdetails}).\fTBD{GKD: i'm sort of repeating myself here because i talk about velocity above without defining it. maybe introduce all this earlier ... or not.}
\fTBD{GKD: consider explaining in more detail and concretely in the appendix. If so, add a cref here.} We look at submatrices of a fixed size, formed by \samples with contiguous \scoreerror scores.  
\cref{fig:ntk} shows that higher \eltoon scores lead to higher velocities. This relationship is not affected by the time at which both are computed.

Interestingly, the kernel velocity drops off sharply for \samples with the very highest scores when label noise is introduced. 
In \cref{app:noisydata}, we showed that dropping these \samples
boosts the accuracy of the final predictor.
We hypothesize that, while the kernel velocity is higher for harder \samples that the model is actively trying to fit, the kernel velocity drops off for the very highest scoring \samples that might be too difficult to learn, perhaps because they are unrepresentative samples or they have have label noise. 

\subsection{Connections to the Linear Mode Connectivity}
\label{sec:lmc}

We now examine how the ranking of the examples by \scoreerror connects to the geometry of the loss surface. In particular, \citet{frankle2020linear} studied the effect of minibatch randomness on the training trajectory, focusing on identifying the point in training
when two networks, starting from the same weights, but trained with independent minibatches,  converge
to the same ``linearly connected'' mode.
They find that, for standard vision datasets, the onset of this ``linear mode connectivity'' (LMC) happens early in training. 

More precisely, let $w_1,w_2,\dots,w_T$ be the training trajectory of a \emph{parent} network,
fix a \emph{spawning time} $t^{*}$,
and let $v_{t^*},v_{t^*+1},v_{t^*+2},\dots,v_T$ be an independent training trajectory (i.e., with independent minibatches), beginning at $v_{t^*}=w_{t^*}$. We call $v_T$ the child network and $v_{t^*},v_{t^*+1},\dots$ the child trajectory.
The (training) error barrier between two weights $w$ and $w'$,
denoted $\err(w,w';S)$,
is 
the maximum deviation of the training error surface $\EmpRisk{S}{\cdot}$
above the line connecting the empirical risk at $w$ and $w'$.
That is,
\[
\textstyle
\err(w,w';S) 
=
\sup_{\alpha\in [0,1]} 
    \big \{ 
       \EmpRisk{S}{\alpha\, w + (1-\alpha)\, w' }  
         -  \alpha\, \EmpRisk{S}{w} - (1-\alpha)\, \EmpRisk{S}{w'} 
    \big \} .
\]
We then define the \emph{mean (training) error barrier, spawning at $t^*$, at time $t$}, for $t^* \le t \le T$,
denoted $\err_t^{t^*}(S)$,
to be the expected error barrier between $w_t$ and $v_t$ on the data $S$. That is, 
\[
\err_t^{t^*}(S) 
=\EE_{w_{t^*+1:t},v_{t^*+1,t}} [ \err(w_t,v_t;S)  ],
\]
where the expectation is taken over the randomness in the trajectories of $w$ and $v$ \emph{after} $t^*$ due
to the choice of minibatches, conditional on the initial trajectories up through time $t^*$.
(Note that, at the end of training $t=T$, the supremum in $\err(w_T,v_T;S)$ is
often achieved near $\alpha=1/2$, and so this is a cheap approximation used in practice.)
The ``onset'' of linear mode connectivity is the earliest spawning time $t^*$ at which point $\err_T^{t^*}(S) \approx 0$, where $S$ is the whole training set. In our work, we instead compute the error barrier on  \emph{subsets of the training set}, which allows us to compare the training dynamics and modes on \emph{subpopulations}.

In \cref{fig:lmc}, we measure the mean error barrier $\err_t^{t^*}(S')$ as a function of the spawning time $t^*$, in the cases where $S'$ are either 1) the training \samples with the smallest scores, 2) the largest scores, or 3) a random subset of training \samples.
We find that the error barrier falls close to zero very rapidly for \samples that have low \scoreerror scores, and stays high for high score \samples.
These findings suggest that the loss landscape derived from restricted subsets of examples with low and high \scoreerror behave very differently.  The loss landscape derived from easy subsets of examples with low scores is quite flat, in the sense that error barriers between children as a function of spawn time rapidly diminish. On the other hand, the loss landscape derived from harder subsets of examples with higher scores is rougher, with higher error barriers that persist for longer in the spawn time.  Further, this result is in agreement with the results presented in \cref{sec:ntk}, showing that most of the learning happens in the high \scoreerror score \samples.   %

\begin{figure}[t]
\includegraphics[width=\linewidth]{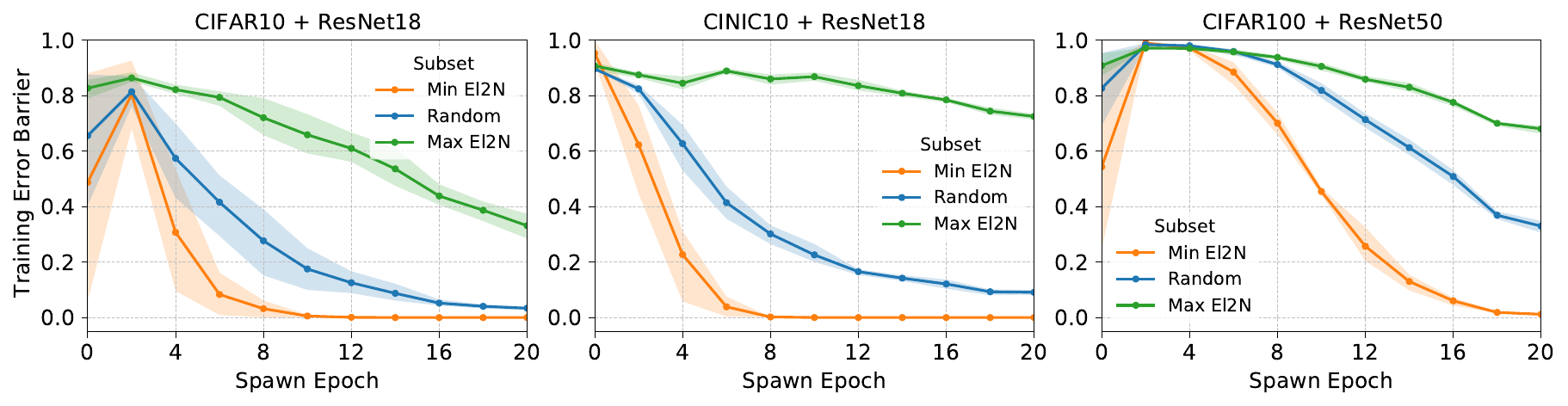}
\centering 
\caption{The final training error barrier between children on subsets of a 1000 highest (\emph{green}) and lowest (\emph{orange}) \scoreerror score \samples, and randomly selected training subset (\emph{blue}) as a function of the spawning time.
\emph{Left to right}: different dataset and network combinations.}
\label{fig:lmc}
\vspace{-1em}
\end{figure}

\section{Related Work}
\label{sec:relatedwork}
\vspace{-1em}
As discussed earlier, our work is closely related to an empirical study by \citet{toneva2018empirical}, 
which examines the frequency with which correct classification decisions are forgotten during training. 
The authors observe that \samples that are rarely forgotten are also ones that do not contribute much to the final accuracy of the predictor. 
In particular, if we retrain from initialization after having removed rarely forgotten \samples from the \trainingdata, we achieve the same accuracy.  Similar to our work, this work analyzes the dynamics of training in deep learning through the lens of training \samples.
However, unlike forgetting scores, our proposed methods use only local information \textit{early in training}. This highlights two properties: example importance is reflected in the local properties of the loss landscape after a few epochs and the ordering of examples by importance is roughly preserved throughout training. We think that this is a key contribution and hope it prompts future empirical and theoretical work exploring the early stage of learning.

\citet{coleman2019selection} use a small proxy network in combination with other training data selection methods to find a small subset of important-for-training \samples, that can then be used to train a large state-of-the-art (SOTA) deep neural network. 
In their empirical study, they observe that most important \samples selected via a proxy model are also important for training a SOTA network.
In addition, they study a proxy which reuses SOTA network's architecture but is trained for a shorter time.
The authors observe that selecting the important examples after at least 50 epochs of training works better than selecting them at random, but not as well as after the full training run. 
They do not study shorter training times for proxies or relate it to the training dynamics in any other way.

Another line of related work is on coresets (see, e.g., \citep{agarwal2005geometric,munteanu2018coresets,huggins2016coresets,campbell2018bayesian,feldman2020turning,tolochinsky2018coresets}, and many others).
The term \emph{coresets} generally refers to a possibly weighted subset of \trainingdata. Much of the work on coresets is focused on identifying small training data subsets that provably yield an $\epsilon$-approximate solution to the original objective (on all the \trainingdata).
Most guarantees require the problem to have special structure, such as convexity. 
For nonconvex problems, like training deep neural networks, guarantees are provided for very conservative proxies, e.g., based on Lipschitz constants or smoothness. 
While coreset selection comes with nice theoretical guarantees, in our opinion, the utility of these methods is best considered an empirical question.

Coresets have also been studied in the active learning community. Here, the goal is to select a small set of \samples to label at any given iteration of training (see, e.g., \citep{wei2015submodularity,sener2017active,killamsetty2020glister, mirzasoleiman2020coresetsA,shen2017deep}, and references therein). 
Coreset selection has also been proposed as a way to increase model robustness \citep{mirzasoleiman2020coresetsB}.

\citeauthor{pleiss2020identifying} use a similar method, the Area Under the Margin (AUM) statistic, but with a slightly different goal: identifying noisy and mislabeled examples. 
Their proposed method exploits differences in the training dynamics of clean and mislabeled samples by keeping track of statistics through the course of training. 
The AUM statistic is similar to forgetting scores in that it uses information from the whole training run. 
In contrast, we focus on \emph{instantaneous information} in the early phase of training. 
In addition to identifying noisy examples, we aim to rank points by importance and therefore also identify redundant/easy examples. 
Thus our approach is complimentary to the AUM statistic and can be used together to obtain higher levels of pruning on noisy datasets.

There have been a number of recent studies looking
at the problem of estimating example difficulty. 
One such approach for identifying difficult examples within a given class is the Variance of Gradients (VoG) score proposed by \citet{agarwal2020estimating}. 
For each image, they calculate the gradient of the activations with respect to the pixels at $K$ different checkpoints over training. The VoG score is the average (over pixels) of the per-pixel variance across these $K$ checkpoints.
The authors conclude that the images that appear more difficult also have a higher VoG score.
Understanding the connections between these two scores calculated in entirely different spaces is an interesting direction for future work.

Another quite different approach estimates example difficulty using prediction depth \citep{baldock2021deep}, which is defined as the first layer at which a k-Nearest Neighbor classifier can correctly classify an example using the representation of the image in all subsequent layers. 
In additional to methodological differences, 
this method uses the final trained network. 
To our knowledge, we are the first to highlight the existence of a strong signal for estimating example difficulty early in training. 

Informally, removing a training \sample from the \trainingdata and not hurting the generalization error suggests that the \sample has small ``influence'' on the test data. 
Influence of the training \samples on test \samples is studied in sample-based explainability \citep{koh2017understanding,barshan2020relatif,pruthi2020estimating}. 
On the theory side, 
\citet{feldman2020does} recently proposed to model data as a mixture of populations
and study the role of memorization when the data distribution is long-tailed.
\citeauthor{feldman2020does} demonstrates conditions under which memorization is necessary for good generalization.
In doing so, he proposes a definition of \sample memorization and influence, 
which can be interpreted as a leave-one-out notion of stability. 
In an empirical study following this work, \citet{feldman2020neural} demonstrate that classifiers trained on computer vision benchmarks benefit from memorization. In particular, training without high-memorization-value \samples comes at a cost of accuracy of the learned neural network classifier.
In \cref{app:comparetomemorization}, we compare \scorename, \scoreerror, forgetting scores, and memorization values on CIFAR-100-trained Resnet50 networks; memorization values do not correlate with the other scores.

\section{Discussion}
\label{sec:discussion}
\vspace{-1em}
In summary, our work both (1) introduces methods to significantly prune data without sacrificing test accuracy using {\it only} local information {\it very early} in training (\cref{fig:tradeoffs}), sometimes even at initialization, 
and (2) uses the resulting methods to obtain new scientific insights into how different subsets of training examples drive the dynamics
of deep learning. 
We start from a principled approach by asking how much on average each training \sample influences the loss reduction of other examples, and from that starting point, we obtain $2$ scores, namely gradient norm (\scorename) and error norm  (\scoreerror) that bound or approximate this influence, with higher scores indicating higher potential influence. We find that \samples with higher scores tend to be harder to learn, in the sense that they are forgotten more often over the entire course of training. We also find that the very highest scoring examples tend to be either unrepresentative outliers of a class, have non standard backgrounds or odd angles, are subject to label noise, or are otherwise difficult.  This observation yields a simple and powerful sliding window method (\cref{fig:slidingwindow}) to prune data by keeping examples within a range of scores, where the start and the end of the range constitute just $2$ hyperparmeters that can be tuned via a validation set. 
This tuning can be done using different hyperparameter settings or on a different network saving computation time (\cref{app:arch_gen}).
Furthermore, we find that high-scoring examples primarily drive feature learning by maximally supporting the velocity of the NTK, whereas learning dynamics might actually give up on the very highest scoring examples that may correspond to unrepresentative examples or noise (\cref{fig:ntk}). 
Finally we show that higher (lower) scoring subsets of examples contribute to a rougher (smoother) loss landscape (\cref{fig:lmc}). Overall this decomposition of both loss landscape geometry and learning dynamics into differential contributions from different types of \samples constitutes an exciting new methodology for analyzing deep learning.  A deeper understanding of the differential role played by different subsets of examples could aid not only in data pruning, but also in curriculum design, active learning, federated learning with privacy, and analysis of fairness and bias. 
Our empirical findings raise a number of interesting theoretical questions, some of which we discuss in \cref{app:ablate}.

\section*{Acknowledgements}
The authors would like to thank Blair Bilodeau, Alex Drouin, \'Etienne Marcotte, and Daniel M. Roy for feedback on drafts, the Toolkit team at ServiceNow for providing the tools and computation resources that greatly accelerated our empirical work, and the NeurIPS reviewers for their thorough engagement and invaluable feedback on label-dependence and practical applications. S.G. thanks the Simons Foundation, NTT Research and an NSF Career award for funding while at Stanford.

\printbibliography

\clearpage

\appendix

\section{Ethical and societal consequences }
\label{sec:ethics}

This work raises several ethical considerations. Being, an empirically driven work, it consumed considerable energy. However, we hope that it will enable advancements in theory that will more efficiently guide experiments. Also, we focus mostly on accuracy as a metric, which tends to hide disparate effects on marginalized groups. But since this work attempts to explicitly uncover the influence of training examples and sub-populations, we hope that it will lead to methods that will decrease bias in the training procedure, especially if marginalized groups are under-represented in the dataset and are thus difficult to learn.

\begin{figure}[t]
\includegraphics[width=\linewidth]{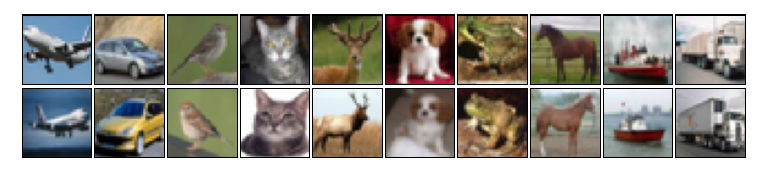}
\centering 
\caption{Examples with the smallest (first row) and second smallest (second row) \grand scores for each class (columns, from left to right: airplane, automobile, bird, cat, deer, dog, frog, horse, ship, truck) from a ResNet18 trained on CIFAR-10. \grand scores were calculated at initialization.} 
\label{fig:min_grans}
\end{figure}

\begin{figure}[t]
\includegraphics[width=\linewidth]{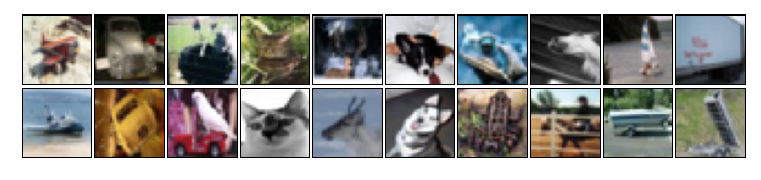}
\centering 
\caption{Examples with the second largest (first row) and  largest (second row) \grand scores for each class (columns, from left to right: airplane, automobile, bird, cat, deer, dog, frog, horse, ship, truck) from a ResNet18 trained on CIFAR-10. \grand scores were calculated at initialization.}
\label{fig:max_grans}
\end{figure}

\begin{figure}[t]
\includegraphics[width=\linewidth]{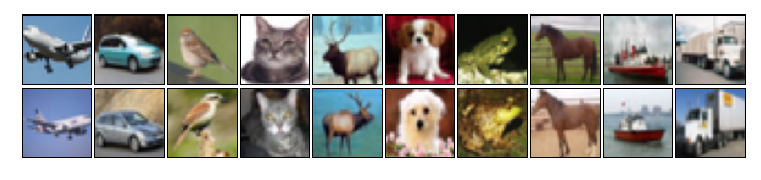}
\centering 
\caption{Examples with the smallest (first row) and second smallest (second row) \eltoon scores for each class (columns, from left to right: airplane, automobile, bird, cat, deer, dog, frog, horse, ship, truck) from a ResNet18 trained on CIFAR-10. \eltoon scores were calculated at epoch 10.}
\label{fig:min_el2ns}
\end{figure}

\begin{figure}[t]
\includegraphics[width=\linewidth]{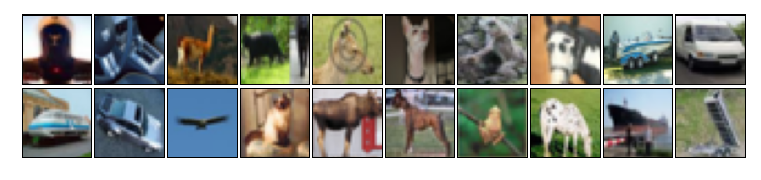}
\centering 
\caption{Examples with the second largest (first row) and  largest (second row) \eltoon scores for each class (columns, from left to right: airplane, automobile, bird, cat, deer, dog, frog, horse, ship, truck) from a ResNet18 trained on CIFAR-10. \eltoon scores were calculated at epoch 10.}
\label{fig:max_el2ns}
\end{figure}

\section{Implementation Details}
\label{app:params}

The code is made available at \url{https://github.com/mansheej/data_diet}

\subsection{Resources used}
\label{sec:resources}
We run all experiments on a single 16GB NVIDIA Tesla V100 GPU. The entire project (from early exploration to final paper) used about 15000 GPU hours. We used an internal cluster at ServiceNow.%

\subsection{Training details}

\paragraph{Deep Learning Frameworks.} We used JAX \citep{jax2018github} and Flax \citep{flax2020github} in our implementations.

\paragraph{Data.}
We use CIFAR-10, CIFAR-100 \citep{krizhevsky2009learning}, and CINIC-10 \citep{cinic10cite}. 
CIFAR-10 and CIFAR-100 are used in their standard format. 
For CINIC-10, we combine the training and validation sets into a single training set with 180000 images.
The standard test set of 90000 images is used for testing.
Each dataset is normalized by its per channel mean and standard deviation over the training set.
All datasets get the same data augmentation: pad by 4 pixels on all sides, random crop to 32$\times$32 pixels, and left-right flip image with probability half.

\paragraph{Models.}
We use ResNet18-v1 and ResNet50-v1. Our implementation is based on the example in Flax \citep{flax2020github} designed for larger high resolution images. Since CIFAR and CINIC images are 32$\times$32 pixels only, we use the low-resolution variant of these networks: the first two layers (a convolution layer with 7$\times$7 kernel and 2$\times$2 stride, and a max pooling layer with 3$\times$3 kernel and 2$\times$2 stride) are replaced with a single convolution layer with 3$\times$3 kernel and 1$\times$1 stride.

\paragraph{Training hyperparameters.}
All networks are trained with the Stochastic Gradient Descent (SGD) optimizer, learning rate = 0.1, nesterov momentum = 0.9, weight decay = 0.0005. 
For CIFAR-10 and CIFAR-100, we use batch size = 128, and for CINIC-10, we use batch size = 256. 
The learning rate is decayed by a factor of 5 after 60, 120 and 160 epochs and all networks are trained for a total of 200 epochs (for the full dataset, i.e. 78000 steps for CIFAR-10 and CIFAR-100, and 140600 steps for CINIC-10).
When using a pruned dataset, to allow for a fair comparison with the full dataset, we keep the number of iterations and schedule fixed for different pruning levels.

\subsection{Experimental details}
\label{kvdetails}

\paragraph{Reporting results.} For every quantity we plot, we do 4 independent runs (independent model initialization and SGD noise) and report the mean and the 16th to 84th percentile of obtained accuracies for representing variability across runs. 
The mean is reported as lines and the variability is reported as shading in the plots.
When evaluating scores, we train new randomly initialized models with different seeds from those used to calculate the scores.

\paragraph{Calculating scores.} All scores (\eltoon, \grand and forgetting scores) are calculated by averaging the scores across 10 independent runs.

\paragraph{Random label experiments.} For the random label experiments, at the beginning of training, we pick 10\% of the examples randomly and permute their labels (to keep overall label statistics fixed). 
The subsets are selected as follows: 
\begin{enumerate}
    \item score all examples and sort them in ascending order by score; 
    \item drop the set of images with the smallest scores that make up a fraction of the dataset equal to the specified offset; 
    \item keep the next set of images based on the given subset size; 
    \item drop all the following images.
\end{enumerate}

\paragraph{Kernel velocity experiments.} The kernel velocities are calculated with 100 examples. 
The examples are picked as follows: first, score all examples and sort them in ascending order by score; then, for every example index in the figure, calculate the NTK gram matrix velocity (described in \cref{sec:ntk}) for 100 contiguous sorted images starting at that example index.

The NTK submatrix velocity on a subset of examples at a particular point in training is defined as follows. 
Let the examples be $x_1, x_2, ..., x_m$. 
Let the time at which the NTK submatrix is calculated be $t$, the parameters of the network $f$ is $\mathbf{w}_t$. 
Let $C$ be the number of classes and $N$ the number of parameters in the network.
Using the notation in \cref{notation}, the $k$th logit gradient of the model on example $i$ is $\psi^{(k)}_t(x_i) = \grad_{\weights_t} f_t^{(k)}(x_i)$. We cast these into a $mC\times N$ matrix $\Psi_t$ where the rows run over each logit gradient of each image and the columns run over the parameters of the model. 
The NTK submatrix is a $mC\times mC$ matrix given by $K_t = \Psi_t\Psi_t^T$.
The kernel velocity is calculated as
\begin{equation}
    v = 1 - \frac{\langle K_t, K_{t+1}\rangle}{\norm{K_t} \norm{K_{t+1}}}
\end{equation}
where $\langle \cdot, \cdot \rangle$ is the Frobenius inner product and $\norm{\cdot}$ is the Frobenius norm.

\paragraph{Linear mode connectivity experiments.} The training error barrier between children is calculated by following \citep{fort2020deep,frankle2020linear}. In addition to estimating the error barrier on 1000 random images, we also estimate the error barrier on 1000 images with the largest and smallest \eltoon score. The \eltoon score used is calculated at epoch 10 of the parent run.

\section{Example Images}
\label{app:imgs}

In this section, we examine the examples with small and large \grand and \eltoon scores for a ResNet18 trained on CIFAR-10. 
\grand scores were computed at initialization and \eltoon scores at epoch 10.
We show two examples from each class with both minimum and maximum \eltoon and \grand scores in \cref{fig:min_grans,fig:max_grans,fig:min_el2ns,fig:max_el2ns}.
The examples with the minimum \grand and \eltoon scores tend to be simple, canonical representations of each class pictured from very typical angles. 
The examples with maximum scores are harder to identify; they are blurrier, from strange angles or have unexpected backgrounds or other artifacts.

\begin{figure}[t]
\includegraphics[width=0.5\linewidth]{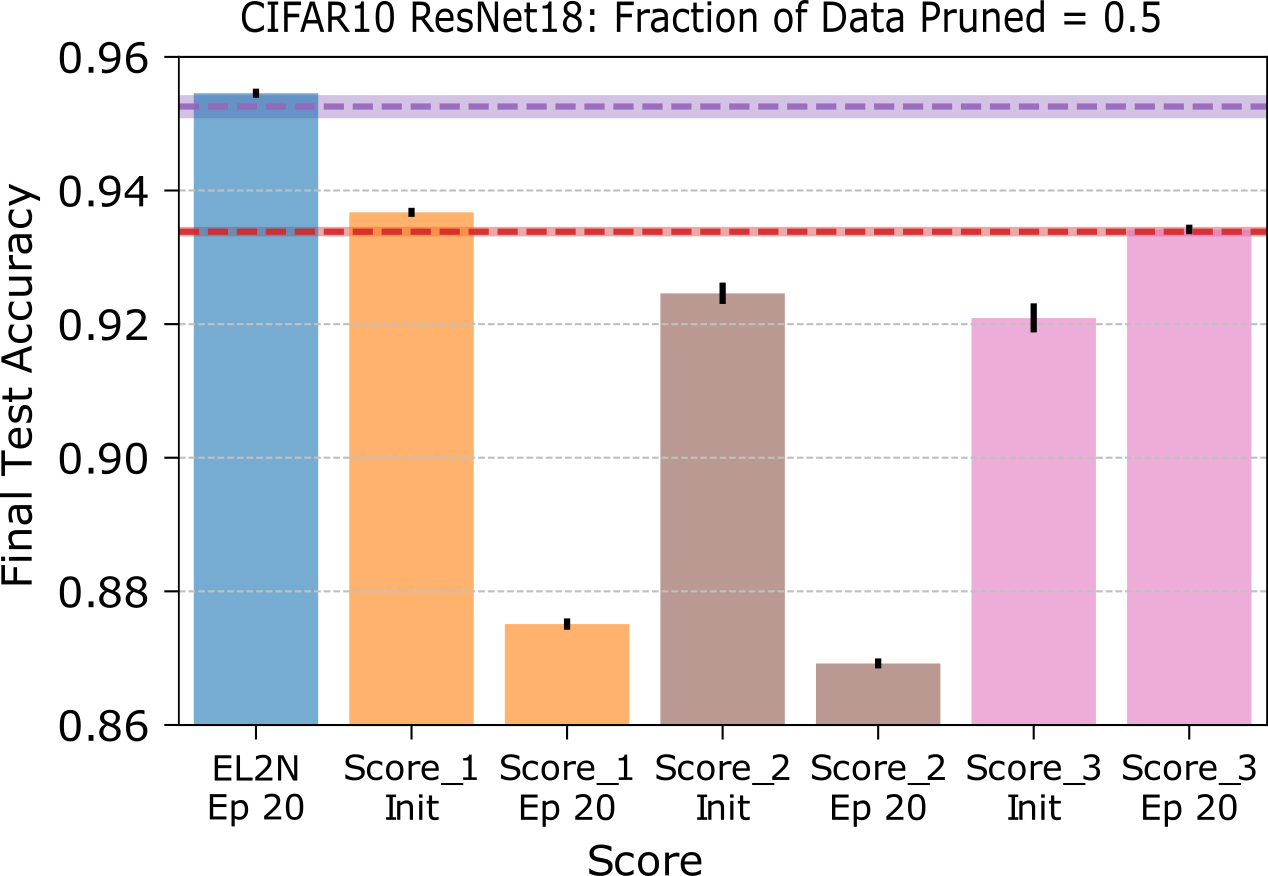}
\centering 
\caption{Compare \eltoon score (label-dependent) to three label-independent variants at initialization and early in training. Experiments performed on ResNet18 trained on CIFAR10. Purple dashed line corresponds to final test accuracy after training on the full training dataset and red dashed line corresponds to final test accuracy after training on a random 50\% subset of the data. Bars are the final test accuracies obtained from training on 50\% of the training data with the highest corresponding score. Definitions of scores 1-3 in \cref{app:ablate}.}
\label{fig:ablate_labels}
\end{figure}

\section{Comparisons to Label Independent Scores}
\label{app:ablate}
For classification tasks, the GraNd score for a training example $(x, y)$, is defined as
\[
\mathrm{GraNd}_t(x, y)=\mathbb{E}\left\|\sum_{k=1}^K \left(p(\weights_t,x)^{(k)} - y_k\right) \psi_t^{(k)}(x) \right\|_2
\label{eq:grandapp}
\]
where $k$ indexes classes. 

At initialization, the weights are label independent, and thus $\psi_t^{(k)}(x)$ has no label information. 
Therefore, the \grand score in theory should not depend on the labels.

Note, however, that in practice we approximate the expectation using a finite number of samples of different initializations.
In this section we ablate the \grand score to tease apart the contribution of the labels, if any, and whether there is any signal in the logit gradient averaged over classes.

In contrast to the \grand score, the EL2N scores consist of just the label prediction error:
\[
\mathrm{EL2N}_t(x, y)=\mathbb{E}\left\|p(\weights_t,x) - y\right\|_2.
\label{eq:el2napp}
\]

We evaluate three different label independent variants, defined as
\[
\mathrm{Score_1}_t(x)= \mathbb{E}\left\|\sum_{k=1}^K\psi_t^{(k)}(x)\right\|_2,
\label{eq:score1}
\] 
\[
\mathrm{Score_2}_t(x)= \mathbb{E}\left\|p(\weights_t,x) - \frac{1}{K}\mathbf{1}\right\|_2,
\label{eq:score2}
\]
\[
\mathrm{Score_3}_t(x)=\mathbb{E}\left[\sum_{k=1}^K\|\psi_t^{(k)}(x)\|_2^2\right].
\label{eq:score3}
\]
Our experimental setup is as follows:
\begin{enumerate}
    \item Train 10 independent ResNet18 networks on CIFAR10.
    \item Calculate the scores in \cref{eq:score1,eq:score2,eq:score3} at both initialization and epoch 20. The expectation is taken across the 10 networks. This gives us 6 new scores: 3 score types at 2 initializations each.
    \item Prune 50\% of the CIFAR10 training examples that have the smallest scores. This gives us 6 new pruned datasets, one for each score.
    \item Train 3 new randomly initialized ResNet18 networks on each of the 6 pruned datasets.
    \item For each of the 6 scores, evaluate their performance by averaging their final test accuracy across the 3 networks trained on the corresponding dataset.
\end{enumerate}
In \cref{fig:ablate_labels} we compare these scores to training with a 50\% subset of 
maximum EL2N scores at epoch 20, a random 50\% subset and the full training set. 
Only the label-dependent scores perform significantly better than the random baseline.

\section{Additional Experiments}
\label{app:additional}

\begin{figure}[t]
\includegraphics[width=0.7\linewidth]{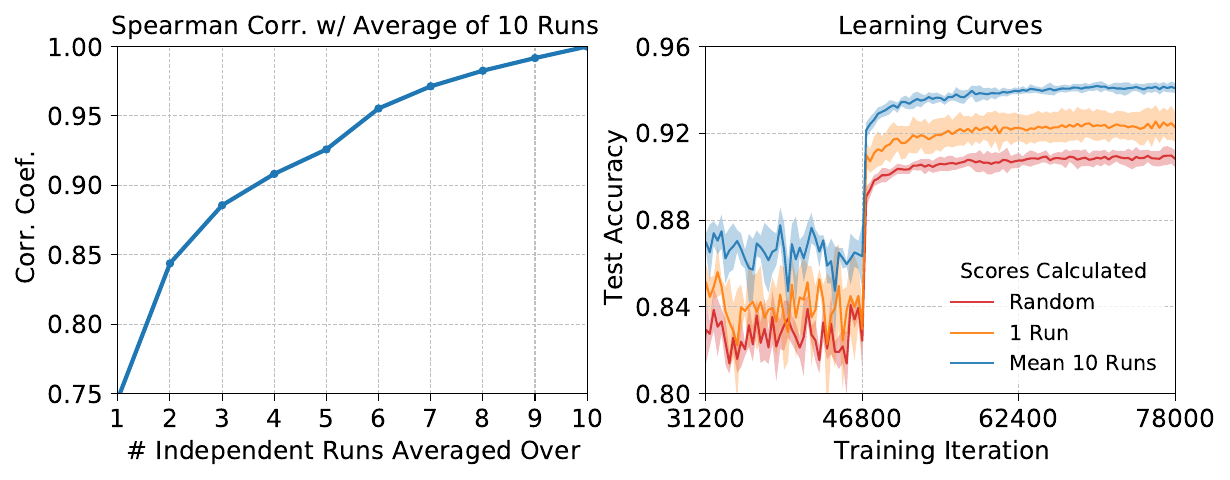}
\centering 
\caption{The effect of averaging \grand scores calculated at initialization of a ResNet18 trained on CIFAR-10. \textit{Left:} The Spearman rank correlation coefficient between \grand scores obtained by averaging over a given number of independent runs (x-axis) to those obtained by averaging over 10 independent runs. \textit{Right:} Training accuracy curves for ResNet18 trained on a 50\% subset of CIFAR-10. The training subset is obtained by either random sampling or keeping the examples with the largest \grand scores. We compare the test accuracy obtained using \grand scores from 1 initialization and from averaging over 10 independent initializations. We zoom in to the end of training to highlight the differences between the learning curves.}
\label{fig:grand_stab}
\end{figure}

\begin{figure}[t]
\includegraphics[width=0.7\linewidth]{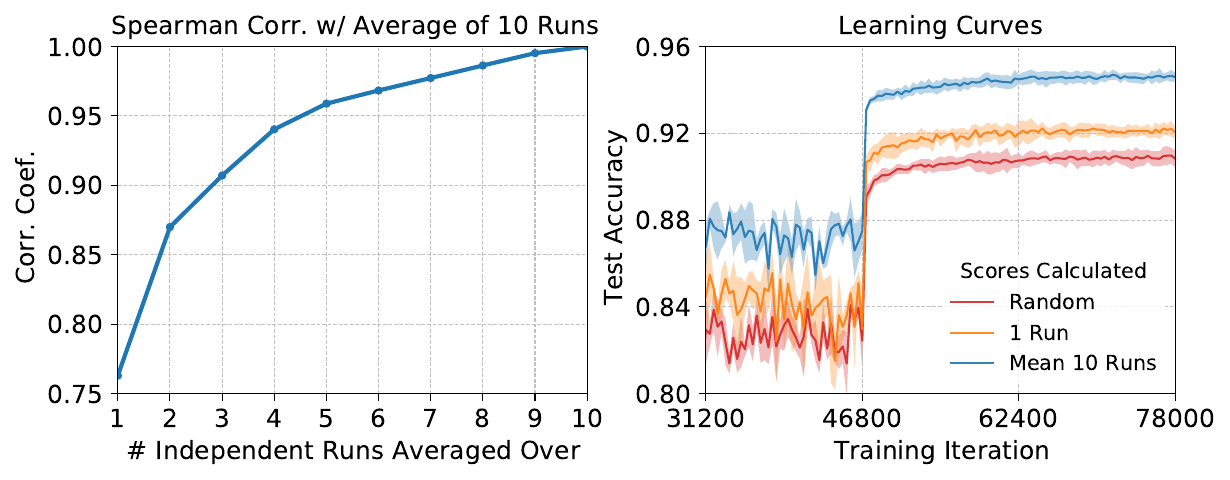}
\centering 
\caption{The effect of averaging \eltoon scores calculated at epoch 10 of a ResNet18 trained on CIFAR-10. \textit{Left:} The Spearman rank correlation coefficient between \eltoon scores obtained by averaging over a given number of independent runs (x-axis) to those obtained by averaging over 10 independent runs. \textit{Right:} Training accuracy curves for ResNet18 trained on a 50\% subset of CIFAR-10. The subset is selected by either random sampling or keeping the examples with the largest \eltoon scores. We compare the test accuracy obtained when pruning based on \eltoon scores from a single run and from averaging over 10 independent runs. We zoom in to the end of training to highlight the differences between the learning curves.}
\label{fig:eltoon_stab}
\end{figure}

\subsection{Sensitivity analysis of \grand and \eltoon scores}
\label{app:details}

In all our experiments, \grand and \eltoon scores are averaged over 10 independent initializations or runs.
This turns out to be essential for successful pruning using the scores.

In \cref{fig:grand_stab} (right) we show the effect of averaging on \grand scores computed at initialization on a ResNet18, CIFAR-10.
On average, \grand scores for any individual run have a Spearman rank correlation of about 0.75 with the \grand scores averaged over 10 runs.

We also compare learning curves when training using 50\% of the training data selected as follows: using high \grand scores at a single initialization; using high \grand scores averaged over 10 initializations; randomly (baseline). 
As seen in \cref{fig:grand_stab} (right), using scores averaged over 10 initializations performs significantly better. 
In \cref{fig:eltoon_stab}, we show similar results for \eltoon scores calculated at epoch 10 on a ResNet18 trained on CIFAR-10.

These results suggest that \grand and \eltoon scores represent properties of the dataset rather than of a specific network weights. 
To get an accurate ranking of example importance, we need to average out the effects of individual initializations/weights.
Empirically, we find that averaging over 10-20 runs suffices, and averaging over more runs has insignificant additional benefit.

\subsection{Comparison between scores from different architectures on the same dataset}
\label{app:resnet50cifar10}

\begin{figure}[t]
\includegraphics[width=0.7\linewidth]{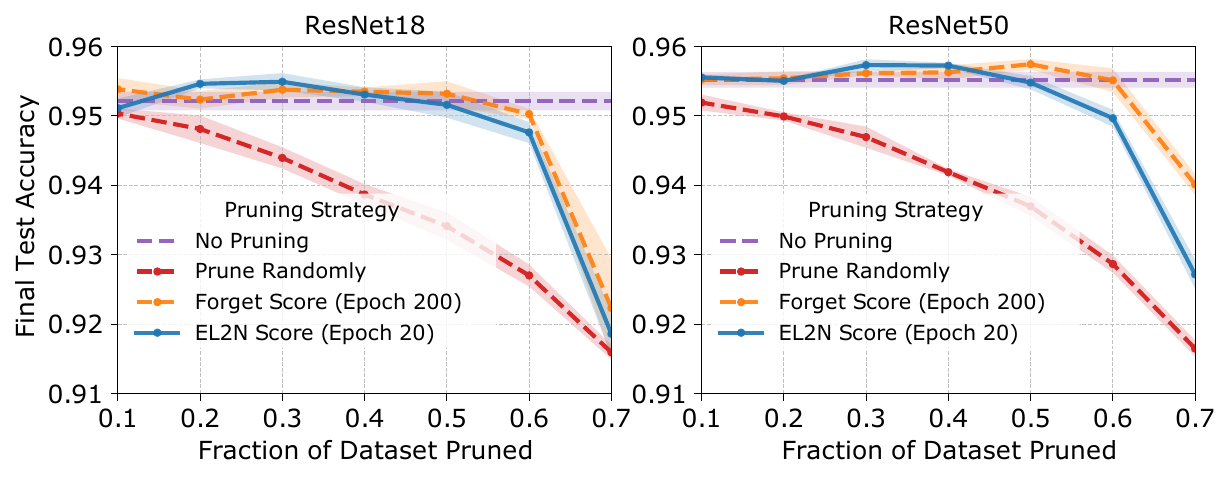}
\centering 
\caption{Experimental setup similar to \cref{fig:tradeoffs}. We compare CIFAR10 data pruning on ResNet18 and ResNet50. 
The y-axis indicates the final test accuracy achieved by training after pruning different fractions of the dataset (x-axis). Compare forgetting scores at the end of training and \scoreerror scores early in training (at epoch 20).
In each case, examples with the lowest scores are pruned and then the networks are trained from initialization on the data that was not pruned.
In both plots accuracy achieved by training on the full dataset and on a random subset of the corresponding size dataset are used as baselines.}
\label{fig:18v50}
\end{figure}

In this section, we examine how the choice of network architecture affects pruning with \eltoon and \grand scores. Specifically, we repeat the experiment in \cref{fig:tradeoffs} bottom row, but for a ResNet18 and a ResNet50 trained on CIFAR-10.
The results are shown in \cref{fig:18v50}. 
Independently of the network architecture tested, data pruning by \eltoon scores computed at epoch 20 is competitive with pruning based on forgetting scores computed at epoch 200.
When trained on either of the networks, pruning based on \grand scores does significantly better than the random baseline. 
Overall, these results suggest that network depth has a small effect on our data-pruning results.

\begin{figure}[t]
\includegraphics[width=0.7\linewidth]{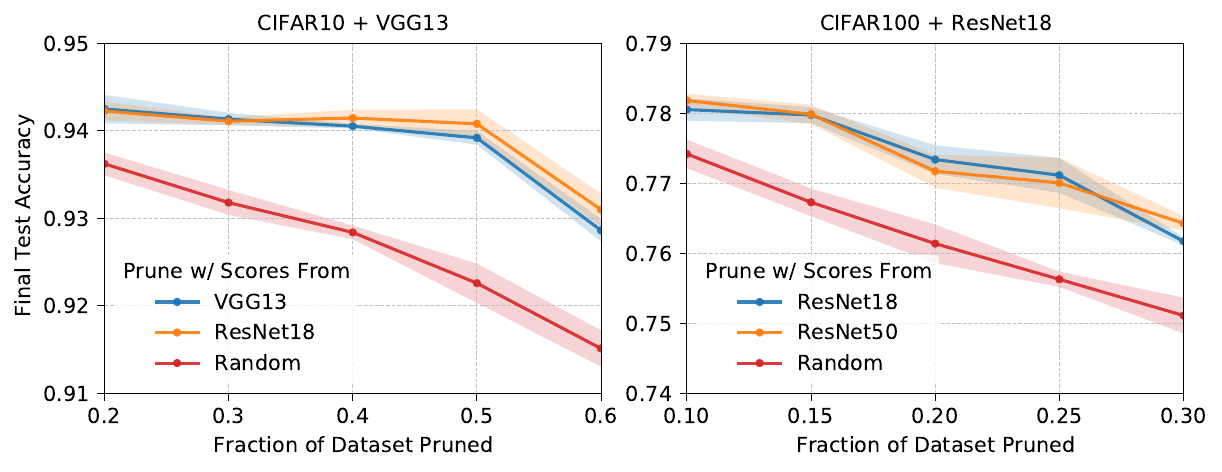}
\centering 
\caption{EL2N scores calculated using one network can be used to prune data for training with another network. Left: we compare VGG13 networks trained on CIFAR10 after pruning different fractions of the train set. The examples with the lowest EL2N scores were dropped. EL2N scores were calculated at epoch 20 using 10 independent networks, with either a VGG13 or a ResNet18 architecture. Right: we compare ResNet18 networks trained on CIFAR100 after pruning different fractions of the train set. The examples with the lowest EL2N scores were dropped. EL2N scores were calculated at epoch 20 using 10 independent networks, with either a ResNet18 or a ResNet50 architecture. Random pruning is used as a baseline. Test accuracy after pruning is agnostic to which network was used to calculate the scores: EL2N scores generalize across architectures.}
\label{fig:arch_gen}
\end{figure}

\subsection{EL2N scores generalize across architectures}
\label{app:arch_gen}
EL2N scores are robust to variations in architectures and seem to capture information intrinsic to the dataset, not network specific behavior. 
To support this, we show that EL2N scores calculated with one architecture can be used to prune the dataset for other architectures; scores calculated from a different architecture leads to final test accuracies that are identical to those trained using scores calculated from the same architecture. Our experiment setup is as follows:
\begin{enumerate}
    \item For CIFAR10, calculate the EL2N scores at epoch 20 using a ResNet18 and a VGG13. This is done by averaging over 10 independently initialized networks.
    \item Use both scores to train new randomly initialized VGG13 networks on n\% subsets of CIFAR10, keeping only the highest scores for each score type. We perform this experiment for a range of n (n = 20, 30, 40, 50, and 60).
\end{enumerate}
VGG13 networks trained on subsets using either score perform identically and significantly better than random. 
To show robustness across architectures and datasets, we repeat this experiment for CIFAR100, ResNet18 instead of VGG13 and ResNet50 instead of ResNet18 with n = 10, 15, 20, 25, 30. 
Our results in \cref{fig:arch_gen} show that, for a given network and datasets, pruning performance of EL2N scores are agnostic to which network is used to calculate them.
In addition to being a tool that extracts something intrinsic about the dataset, this has practical applications as, for a given dataset, one could extract a smaller important subset for future networks to train on, thereby speeding up the iteration process.

\begin{figure}[t]
\includegraphics[width=\linewidth]{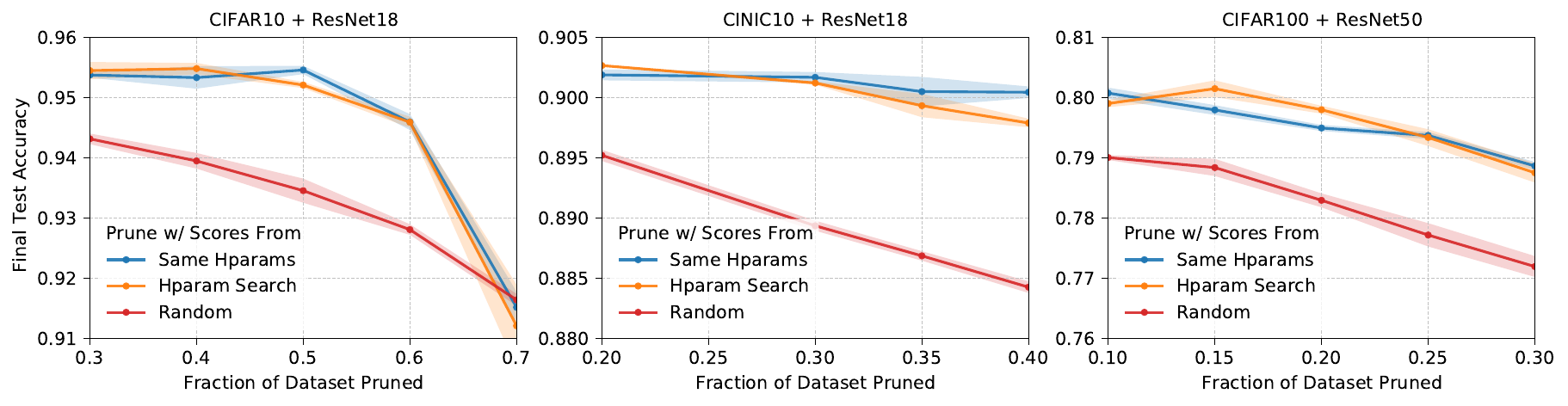}
\centering 
\caption{EL2N scores can be effectively calculated using networks from hyperparameter search runs. In this experiment, we calculate EL2N scores by averaging over networks trained with different hyperparameter configurations. Hyperparameters were chosen for a grid search. We compare this to EL2N scores calculated by averaging over networks trained with the optimal hyperparameters. In both cases, scores were calculated at epoch 20. The scores are evaluated by comparing the final accuracy of networks trained on datasets pruned by either score to different sizes. The evaluation networks were trained with optimal hyperparameters. We also compare to a random pruning baseline. From left to right, we perform this experiment for CIFAR10 with ResNet18, CINIC10 with ResNet18 and CIFAR100 with ResNet50. In all cases, for a wide range of pruning percentages, EL2N scores calculated during hyperparameter optimization perform as well as scores calculated using optimal hyperparameters.}
\label{fig:hpo}
\end{figure}

\subsection{Calculating EL2N scores during hyperparameter optimization}
\label{app:hpo}
For already benchmarked datasets, it can be helpful to reduce the size of the dataset for training future networks.
But in practice, we often have to train on new datasets, and for each new dataset, we need to perform a hyperparameter optimization search. 
In the next experiment, we show that networks trained during hyperparameter optimization can be used to calculate EL2N scores that perform as well as EL2N scores calculated by averaging over an ensemble of networks with optimal hyperparameters. This is a robust effect across architectures and datasets, we show it for ResNet18 traind on CIFAR10, ResNet18 trained on CINIC10, and ResNet50 trained on CIFAR100. Our experiment setup is as follows:
\begin{enumerate}
    \item For each architecture and dataset configuration perform a hyperparameter grid search over learning rates in $\{0.2, 0.1, 0.05\}$ and weight decays in $\{0.001, 0.0005, 0.0001\}$. This leads to 9 training runs using different hyperparameters for each configuration.
    \item For each dataset, calculate EL2N scores at epoch 20 by averaging over the 9 corresponding networks. 
    \item For each dataset, prune a range of fractions of the data by dropping examples with the lowest corresponding score.
    \item For each architecture and dataset configuration, train new randomly initialized networks with optimal hyperparameters on the pruned datasets and evaluate their performance.
\end{enumerate}
As baselines, we compare to random subset baselines and to EL2N scores at Epoch 20 calculated by averaging over 9 networks trained with optimal hyperparameters. Our results are in \cref{fig:hpo}. 
For all three architecture and dataset configurations, and for a range of data pruning levels, EL2N scores calculated during hyperparameter optimization perform as well as EL2N scores calculated using networks with optimal hyperparameters. 
This further demonstrates that EL2N scores capture intrinsic dataset properties and not network specific properties. 
It also has the nice benefit that, after hyperparameter optimization, we essentially get EL2N scores for the price of one forward pass through the whole dataset, thus reusing the expensive compute from hyperparameter search to reduce future compute and make network training more energy efficient.
An interesting direction for future research is to cleverly use scores to identify smaller datasets that can be used to speed up the hyperparameter optimization process. 

\begin{figure}[t]
\includegraphics[width=0.4\linewidth]{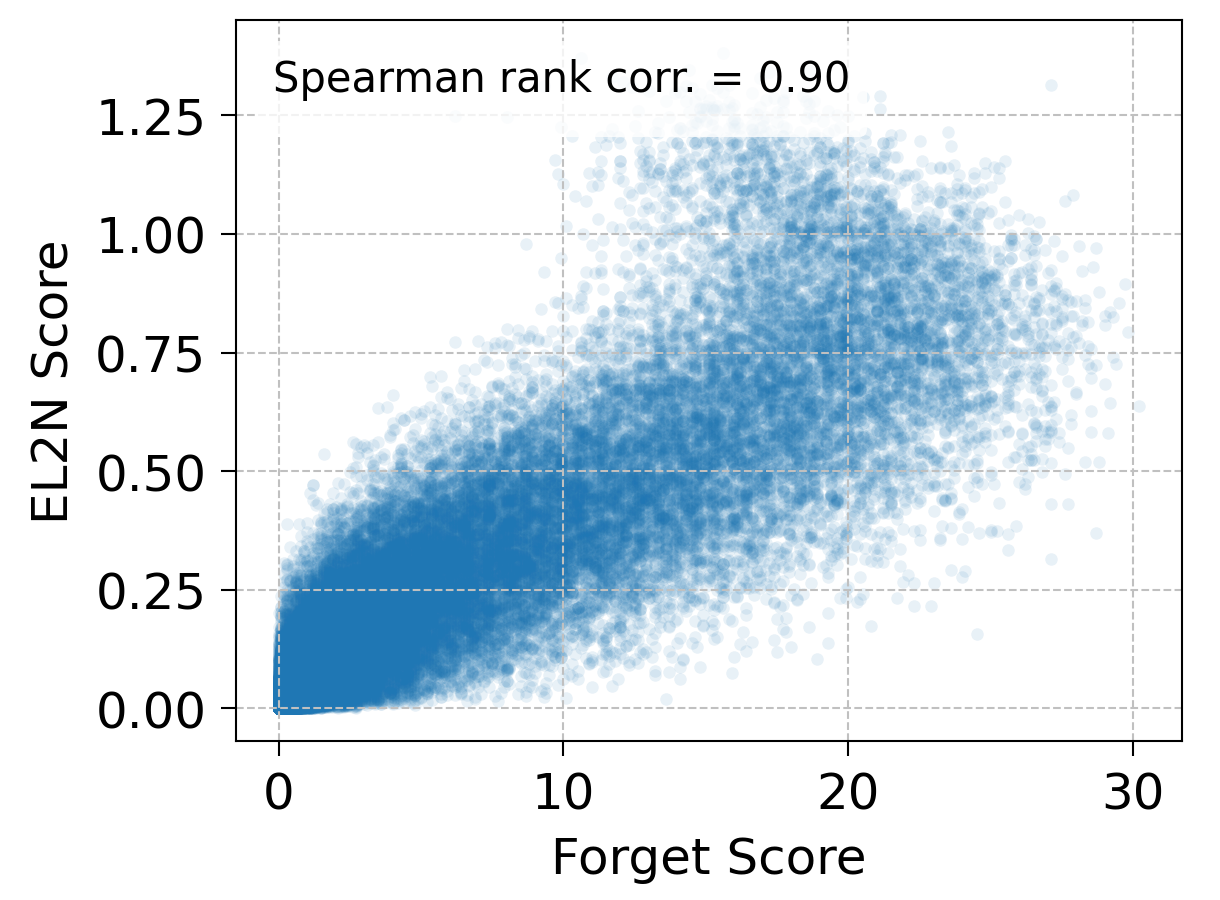}
\centering 
\caption{For a ResNet18 trained on CIFAR-10, we compare the values and Spearman rank correlations for pairs of scores. The scores compared are \grand scores at initialization, \eltoon scores at epoch 20 and forgetting scores at epoch 200.}
\label{fig:coff_cifar10}
\end{figure}

\subsection{Correlations between scores}
\label{app:correlation}

As discussed in previous sections, we find that different scores lead to similar pruning results. In \cref{fig:coff_cifar10}, for a ResNet18 trained on CIFAR-10, we compare the values and Spearman rank correlations for pairs of scores. \grand scores are computed at initialization, \eltoon scores at epoch 20 and forgetting scores at epoch 200. 
\eltoon and forgetting scores, which have the most similar performance, have the highest Spearman rank correlation.

\section{Noise}
\label{app:noisylabels}

\subsection{Noisy Examples in Low Pruning Regime}
\label{app:noise}

\begin{figure}[t]
\includegraphics[width=0.7\linewidth]{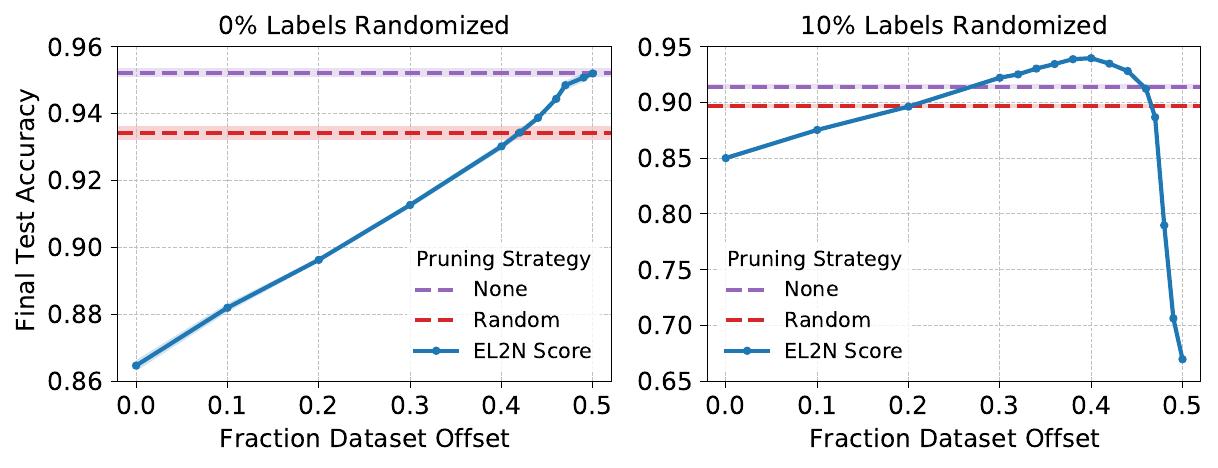}
\centering 
\caption{ResNet18 trained on a 50\% subset of CIFAR-10 with clean \textit{(left)} and 10\% randomized labels \textit{(right)}. 
The training subset contains the \emph{lowest} scoring \samples
\emph{after} \samples with scores below the offset (x-axis) are discarded.
Scores computed at epoch 10.}
\label{fig:slidingwindow50}
\end{figure}

We repeat the noisy labels experiment of \cref{fig:slidingwindow} in a different regime; instead of pruning 60\% of the data, we prune just 50\% of the data. Results are shown in \cref{fig:slidingwindow50}. For the experiment with no labels randomized, at this lower pruning level we no longer see a boost in performance from dropping the highest score examples. However we do see decreasing marginal gains. This suggests that when we have enough data, keeping the high score examples, which are often noisy or  difficult, does not hurt performance and can only help.

\subsection{Scores for Noisy Examples}
\label{app:noisescore}

\begin{figure}[t]
\includegraphics[width=0.7\linewidth]{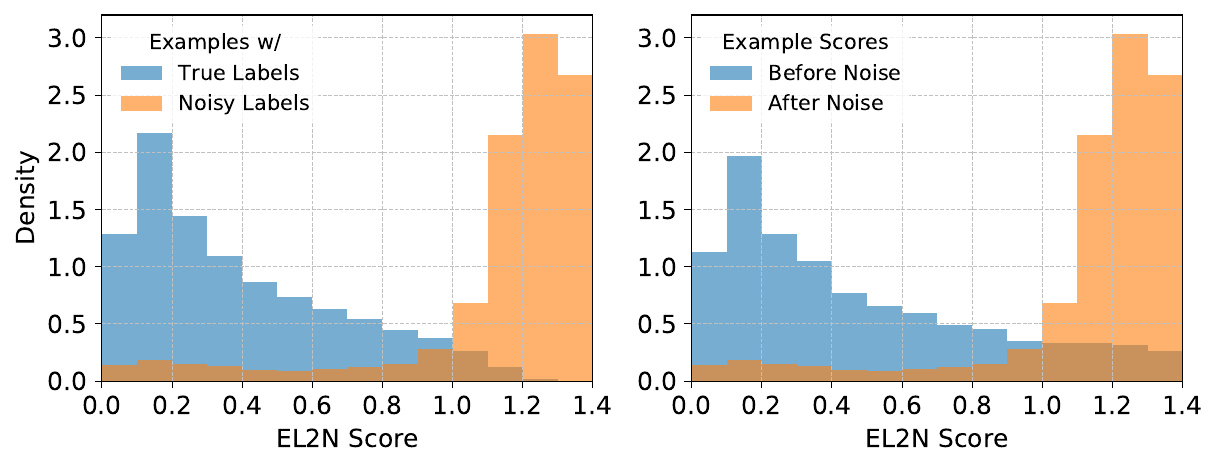}
\centering 
\caption{Scores for examples in the noisy label experiment of \cref{fig:slidingwindow} 
\textit{Left.} Distribution of \eltoon scores for examples with true labels and corrupted labels. \textit{Right.} Distribution of \eltoon scores for examples that are corrupted before and after noise is added.}
\label{fig:noisescores}
\end{figure}

We now examine how adding noise affects the \eltoon scores of examples. This analysis is done for the experiment in \cref{fig:slidingwindow}. See \cref{app:noisydata} for details. Results are shown in \cref{fig:noisescores}. Two results suggest that \eltoon scores can be used to identify images with corrupted labels. First, the \eltoon scores of images with corrupted labels tend to be higher than those with regular labels. Second, after an example's label is corrupted, its \eltoon scores tends to be larger than before.

\section{Comparison to memorization threshold}
\label{app:comparetomemorization}

\begin{figure}[]
\includegraphics[width=0.4\linewidth]{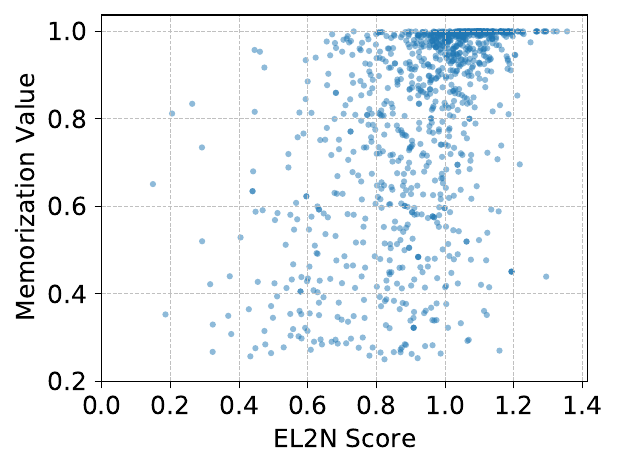}
\centering 
\caption{Comparison of \eltoon scores and Memorization values from \citep{feldman2020neural}}
\label{fig:mem}
\end{figure}

In \cref{fig:mem} we compare \eltoon scores to the memorization values defined in \citep{feldman2020neural}. The memorization values for 1015 CIFAR100 examples are provided by the authors. We replicate their setting by training a ResNet50 on CIFAR-100 and compute the \eltoon scores at epoch 20 for the 1015 examples they provide. 
As seen from \cref{fig:mem}, memorization values and \eltoon scores do not appear to be correlated.

\end{document}

%% file: headers.tex
\usepackage[utf8]{inputenc} 
\usepackage{amsmath, amssymb,bm, cases, mathtools, thmtools}
\usepackage{verbatim}
\usepackage{graphicx}\graphicspath{{figures/}}
\usepackage{multicol}
\usepackage{tabularx}
\usepackage[usenames,dvipsnames]{xcolor}
\usepackage{mathrsfs} 
\usepackage{caption}
\usepackage{algorithm}
\usepackage{algorithmicx}
\usepackage[noend]{algpseudocode}
\usepackage{xspace}

\usepackage[%
    minnames=1,maxnames=99,maxcitenames=3,
    style=numeric-comp,%
    sorting=none,
    doi=false,url=false,
    giveninits=true,
    hyperref,natbib,backend=biber]{biblatex}
\renewbibmacro{in:}{%
  \ifentrytype{article}{}{\printtext{\bibstring{in}\intitlepunct}}}
\bibliography{biblio}

\usepackage[colorlinks,citecolor=blue,urlcolor=blue,linkcolor=RawSienna]{hyperref}
\usepackage{hypernat}
\usepackage{datetime}

\DeclareMathAlphabet\EuRoman{U}{eur}{m}{n}
\SetMathAlphabet\EuRoman{bold}{U}{eur}{b}{n}

\usepackage[capitalize]{cleveref}

\crefname{lemma}{Lemma}{Lemmas}
\crefname{corollary}{Corollary}{Corollaries}
\crefname{theorem}{Theorem}{Theorems}
\crefname{assumption}{Assumption}{Assumptions}

\makeatletter
\let\reftagform@=\tagform@
\def\tagform@#1{\maketag@@@{\ignorespaces\textcolor{gray}{(#1)}\unskip\@@italiccorr}}
\renewcommand{\eqref}[1]{\textup{\reftagform@{\ref{#1}}}}
\makeatother

\declaretheorem[style=plain,numberwithin=section,name=Theorem]{theorem}
\declaretheorem[style=plain,sibling=theorem,name=Lemma]{lemma}

\declaretheorem[style=definition,sibling=theorem,name=Definition]{definition}

\numberwithin{theorem}{section}

%% file: defs.tex
\def\[#1\]{\begin{align}#1\end{align}}
\def\*[#1\]{\begin{align*}#1\end{align*}}

\newcommand{\weights}{\mathbf{w}}

\newcommand{\parspace}{\mathcal{W}}
\newcommand{\Dist}{\mathcal D}
\newcommand\optparen[1]{\ifthenelse{\equal{#1}{}}{}{(#1)}}
\newcommand{\RiskChar}{R}

\newcommand{\EmpRisk}[2]{\hat \RiskChar_{#1}\optparen{#2}}

\newcommand{\Reals}{\mathbb{R}}

\newcommand{\grad}{\nabla}

\newcommand{\AND}{\wedge}

\newcommand{\dee}{\mathrm{d}}

\DeclareMathOperator*{\newlim}{\mathrm{lim}\vphantom{\mathrm{infsup}}}

\DeclareMathOperator*{\newsup}{\mathrm{sup}\vphantom{\mathrm{infsup}}}
\renewcommand{\lim}{\newlim}

\renewcommand{\sup}{\newsup}

\def\EE{\mathbb{E}}

\newcommand{\norm}[1]{\lVert #1 \rVert}

\newcommand{\iid}{i.i.d.\xspace}

\newcommand{\scorename}{GraNd\xspace}
\newcommand{\scoreerror}{EL2N\xspace}
\newcommand{\grand}{\scorename}
\newcommand{\eltoon}{\scoreerror}
\newcommand{\sample}{example\xspace}
\newcommand{\samples}{examples\xspace}
\newcommand{\trainingdata}{training data\xspace}

\newcommand{\score}[1]{\chi_{#1}}
\newcommand{\feature}{\psi}

\newcommand{\loss}{\ell}
\newcommand{\err}{\mathrm{err}}

\newcommand{\e}{\mathrm{e}}